\documentclass[letterpaper]{article}
\usepackage[table]{xcolor}
\usepackage{amsfonts}
\usepackage{amsmath}
\usepackage{fullpage}
\usepackage{amsthm}
\usepackage{url}
\usepackage{times}
\usepackage{helvet}
\usepackage{courier}
\usepackage{graphicx}
\newtheorem{lemma}{Lemma}
\newtheorem{proposition}{Proposition}
\begin{document}
\title{Exploiting an Oracle that Reports AUC Scores\\in Machine Learning Contests}
\author{Jacob Whitehill\\
Office of the Vice Provost for Advances in Learning\\
Harvard University\\
\begin{tt}jacob\_whitehill@harvard.edu\end{tt}
}
\maketitle
\begin{abstract}
\begin{quote}
In machine learning contests such as the ImageNet Large Scale Visual Recognition Challenge \cite{ImageNet2015} and the KDD Cup, contestants can submit candidate solutions 
and receive from an oracle (typically the organizers of the competition) the accuracy of their guesses compared to
the ground-truth labels. One of the most commonly used accuracy metrics for binary classification tasks is
the Area Under the Receiver Operating Characteristics Curve (AUC). In this paper we provide proofs-of-concept of
how knowledge of the AUC of a set of guesses can be used, in two different kinds of attacks, to improve the accuracy
of those guesses.  On the other hand, we also
demonstrate the intractability of one kind of AUC exploit by proving
that the number of possible binary labelings of $n$ examples for which a candidate solution obtains a AUC score of $c$
grows exponentially in $n$, for every $c\in (0,1)$.
\end{quote}
\end{abstract}

\section{Introduction}
Machine learning and data-mining competitions such as the ImageNet Large Scale Visual Recognition Challenge \cite{ImageNet2015}, KDD Cup \cite{KDDCup2015},
and Facial Expression Recognition and Analysis (FERA) Challenge \cite{ValstarEtAl2015} have helped to advance the state-of-the-art of
machine learning, deep learning, and computer vision research. By establishing common benchmarks and setting a clearer boundary 
between training and testing datasets -- participants typically never gain access to the testing labels directly -- these competitions
help researchers to estimate the performance of their classifiers more reliably. However, the benefit of such contests depends on
the integrity of the competition and the generalizability of performance to real-world contexts.
If contestants could somehow ``hack'' the competition to learn illegitimately the
labels of the test set and increase their scores, then the value of the contest would diminish greatly.

One potential window that data-mining contestants could exploit is the accuracy ``oracle'' set up by the competition
organizers to give contestants a running estimate of their classifier's performance: In many data-mining contests (e.g., those organized
by Kaggle), contestants may submit, up to a fixed maximum number of times per day, a set of guesses for the labels in the testing set.
The oracle will then reply with the accuracy of the contestant's guesses on a (possibly randomized) subset of the test data.
There are several ways in which these oracle answers can be exploited, including: (1)
If some contestants were able to circumvent the daily maximum number of accuracy queries they can submit to the oracle
(e.g., by illegally registering multiple accounts on the competition's website),
then they could use those extra oracle responses to perform additional parameter or hyperparameter optimization over the \emph{test set}
and potentially gain a competitive edge. (2) The accuracy reported by the oracle, even though it is an aggregate performance metric over the entire 
test set (or a large subset), could convey information about the labels of \emph{individual examples} in the test set. Contestants could use this information to refine
their guesses about the test labels. Given that the difference in accuracy between contestants is often tiny -- 
in KDD Cup 2015, for example, the \#1 and \#2 contestants' accuracies differed by $0.0003$ AUC \cite{KDDCup2015} --
even a small edge achieved by exploiting the AUC oracle can be perceived as worthwhile.

To date, attacks of type (1) have already been implemented and documented \cite{Simonite2015}.
However, to the best of our knowledge, attacks of type (2) have not previously been investigated.
In this paper, we consider how an attacker could orchestrate attacks of type (2) on an oracle that reports
the Area Under the Receiver Operating Characteristics Curve (AUC), which is one of the most common performance metrics for binary classifiers.
Specifically, we make the following {\bf contributions}:
\begin{itemize}
\item[(a)] We describe an attack whereby an attacker whose classifier already achieves high AUC 
and who knows the prevalence of each class can use the oracle to
infer the labels of a few test examples with complete certainty.
\item[(b)] We provide a proof-of-concept of how the AUC score $c$ of a set of guesses
constrains the set of possible labelings of the entire test set, and how this information can be harnessed, using standard Bayesian inference,
to improve the accuracy of those guesses.
\item[(c)] We show that a brute-force attack of type (b) above is computationally tractable only for very small datasets. Specifically, we prove
that the number of possible binary labelings of $n$ examples for which a candidate solution obtains a AUC score of $c$
grows exponentially in $n$, for every $c\in (0,1)$.
\end{itemize}
As the importance and prominence of data-mining competitions continue to increase, attackers will find more and more ingenious methods of
hacking them. The greater goal of this paper is to raise awareness of this potential danger.

\section{Related Work}
Our work is related to the problem of \emph{data leakage}, which is 
the inadvertent introduction of information about test data into the training dataset of data-mining competitions.
Leakage has been named one of the most important data-mining mistakes \cite{MinerEtAl2009}
and can be surprisingly insidious and difficult to identify and prevent \cite{KohaviEtAl2000,KaufmanEtAl2012}. Leakage has traditionally been
defined in a ``static'' sense, e.g., an artifact of the data preparation process that causes certain features
to reveal the target label with complete certainty. The exploitation of an AUC
oracle can be seen as a form of ``dynamic'' leakage: the oracle's AUC response during the competition
to a set of guesses submitted by a contestant can divulge the identity of 
particular test labels or at least constrain the set of possible labelings.

Our research also relates to privacy-preserving machine learning and differential privacy
(e.g., \cite{Dwork2011,ChaudhuriMonteleoni2009,BlumEtAl2013}), which are concerned with how to provide
useful aggregate statistics about a dataset -- e.g., the mean value of a particular attribute, or even a hyperplane to be used for classifying the data -- 
without disclosing private information about particular examples within the dataset.
\cite{StoddardEtAl2014}, for example, proposed an algorithm for computing ``private ROC'' curves and associated
AUC statistics.
The prior work most similar to ours is by \cite{Matthews2013}, who show how an attacker who already knows most of the test labels can
estimate the remaining labels if he/she gains access to an empirical ROC curve, i.e., 
a set of classifier thresholds and corresponding true positive and false positive rates. In a simulation on 100 samples, they show how a simple
Markov-chain Monte Carlo algorithm can recover the  remaining $10\%$ of missing test labels, with high accuracy, if $90\%$ of the test labels are already known.

\section{ROC, AUC, and 2AFC}
\label{sec:roc}
One of the most common ways to quantify the performance of a binary classifier 
is the Receiver Operating Characteristics (ROC) curve. Suppose a particular test set has ground-truth labels
$y_1,\ldots,y_n \in \{0,1\}$ and the classifier's guesses for these labels are $\hat{y}_1, \ldots, \hat{y}_n \in \mathbb{R}$.
For any fixed threshold $\theta \in \mathbb{R}$,
the false positive rate of these guesses is $\textrm{FP}(\theta) \doteq \frac{1}{n_0} \sum_{i\in\mathcal{Y}^-} \mathbb{I}[\hat{y}_i\geq\theta]$ and
the true positive rate is $\textrm{TP}(\theta) \doteq \frac{1}{n_1} \sum_{i\in\mathcal{Y}^+} \mathbb{I}[\hat{y}_i\geq\theta]$,
where $\mathbb{I}[\cdot] \in \{0,1\}$ is an indicator function, $\mathcal{Y}^- = \{ i: y_i=0 \}$ and
$\mathcal{Y}^+ = \{ i: y_i=1 \}$ are the index sets of the negatively and positively labeled examples, and $n_0 = |\mathcal{Y}^-|, n_1 = |\mathcal{Y}^+|$.
The ROC curve is constructed by plotting $(\textrm{FP}(\theta), \textrm{TP}(\theta))$ for all possible $\theta$. 
The Area Under the ROC Curve (AUC) is then the integral of the ROC curve over the interval $[0,1]$.


An alternative but equivalent interpretation of the AUC \cite{Tyler2000,AgarwalEtAl2005}, which we use in this paper, is that it is equal to the probability
of correct response in a 2-Alternative Forced-Choice (2AFC) task, whereby the classifier's real-valued outputs are used to discriminate between one positively
labeled and one negatively labeled example drawn from the set of all such pairs in the test set. If the AUC is $1$, then the classifier can discriminate
between a positive and a negative example perfectly (i.e., with probability $1$). A classifier that guesses randomly has AUC of $0.5$.
Using this probabilistic definition, given the ground-truth labels
$y_{1:n}\doteq y_1,\ldots,y_n$ and the classifier's real-valued guesses $\hat{y}_{1:n} \doteq \hat{y}_1, \ldots, \hat{y}_n$, we can define the 
function $f$ to compute the AUC of the guesses as:
\begin{eqnarray*}
\lefteqn{f(y_{1:n},\hat{y}_{1:n})} && \\
  &\doteq& \frac{1}{n_0 n_1} \sum_{i\in\mathcal{Y}^-} \sum_{j\in\mathcal{Y}^+} \mathbb{I}[\hat{y}_i < \hat{y}_j] + \frac{1}{2}\mathbb{I}[\hat{y}_i = \hat{y}_j]
\end{eqnarray*}
In this paper, we will assume that the classifier's guesses 
$\hat{y}_1,\ldots,\hat{y}_n$ are all distinct (i.e., $\hat{y}_i=\hat{y}_j \iff i=j$), which in many classification problems is common.
In this case, the formula above simplifies to:
\begin{eqnarray}
\label{eqn:auc}
f(y_{1:n},\hat{y}_{1:n}) &\doteq& \frac{1}{n_0 n_1} \sum_{i\in\mathcal{Y}^-} \sum_{j\in\mathcal{Y}^+} \mathbb{I}[\hat{y}_i < \hat{y}_j]
\end{eqnarray}
As is evident in Eq.~\ref{eqn:auc}, all that matters
to the AUC is the \emph{relative ordering} of the $\hat{y}_i$, not their exact values. 
Also, if all examples belong to the same class and either $n_1=0$ or $n_0=0$, then the AUC is undefined.
Finally, the AUC is a \emph{rational} number because it can be written as the fraction of two integers $p$ and $q$.
We use these facts later in the paper.

\section{Exploiting the AUC Score: A Simple Example}
As a first example of how knowing the AUC of a set of guesses can reveal the ground-truth test labels, consider
a hypothetical tiny test set of just 4 examples such that $y_i\in \{0,1\}$ is the ground-truth label for example $i \in \{1,2,3,4\}$.
Suppose a contestant estimates, through some machine learning process, that the probabilities that these examples belong to the positive class are
$\hat{y}_1=0.5$, $\hat{y}_2=0.6$, $\hat{y}_3=0.9$, and $\hat{y}_4=0.4$, and that the contestant submits these guesses to the oracle.
If the oracle replies that these guesses have an AUC of $0.75$, then the contestant can conclude with
complete certainty that the true solution is $y_1=1$, $y_2=0$, $y_3=1$, and $y_4=0$ because this is the \emph{only} ground-truth
labeling satisfying Eq.~\ref{eqn:auc} (see Table \ref{tbl:simple_example}).
The contestant can then revise his/her guesses based on the information returned by the oracle and receive a perfect score. Interestingly,
in this example, the fact that the AUC is $0.75$ is even more informative than if the AUC of the initial guesses were $1$, because the latter
is satisfied by three different ground-truth labelings.
\begin{table}
\begin{center}
\begin{tabular}{c|c|c|c||r}
\multicolumn{5}{c}{\bf AUC for different labelings}\\\hline
$y_1$ & $y_2$ & $y_3$ & $y_4$ & AUC \\ \hline
0 & 0 & 0 & 0 & -- \\
0 & 0 & 0 & 1 & $0.00$ \\
0 & 0 & 1 & 0 & $1.00$ \\
0 & 0 & 1 & 1 & $0.50$ \\
0 & 1 & 0 & 0 & $\approx 0.67$ \\
0 & 1 & 0 & 1 & $0.25$ \\
0 & 1 & 1 & 0 & $1.00$ \\
0 & 1 & 1 & 1 & $\approx 0.67$ \\
1 & 0 & 0 & 0 & $\approx 0.33$ \\
1 & 0 & 0 & 1 & $0.00$ \\
\cellcolor{blue!25}1 & \cellcolor{blue!25}0 & \cellcolor{blue!25}1 & \cellcolor{blue!25}0 & \cellcolor{blue!25}$0.75$ \\
1 & 0 & 1 & 1 & $\approx 0.33$ \\
1 & 1 & 0 & 0 & $0.50$ \\
1 & 1 & 0 & 1 & $0.00$ \\
1 & 1 & 1 & 0 & $1.00$ \\
1 & 1 & 1 & 1 & -- \\
\end{tabular}
\end{center}
\caption{Accuracy (AUC) achieved when a contestant's real-valued guesses of the test labels are
$\hat{y}_1=0.5, \hat{y}_2=0.6, \hat{y}_3=0.9, \hat{y}_4=0.4$, shown for each possible ground-truth labeling.
Only for one possible labeling (highlighted) do the contestant's guesses achieve AUC of exactly $0.75$.
}
\label{tbl:simple_example}
\end{table}

This simple example raises the question of whether knowledge of the AUC could be exploited in more general cases. In the next
sections we explore two possible attacks that a contestant might wage.


\section{Attack 1: Deducing Highest/Lowest-Ranked Labels with Certainty}
In this section we describe how a contestant, whose guesses are already very accurate (AUC close to $1$), can orchestrate an attack
to infer a few of the test set labels with complete certainty. This could be useful for several reasons:
(1) Even though the attacker's AUC is close to $1$, he/she may not know the actual test set labels -- see Table \ref{tbl:simple_example}
for an example. If the same test examples were ever re-used in a subsequent competition, then knowing their labels could be helpful.
(2) Once the attacker knows some of the test labels with certainty, he/she can now use these examples for \emph{training}.
This can be especially beneficial when the test set is drawn from a different population than the training set (e.g., different
set of human subjects' face images \cite{ValstarEtAl2011}).
(3) If multiple attackers all score a high AUC but have very different sets of guesses, then they could potentially collude to infer the labels
of a large number of test examples.

The attack requires rather strong prior knowledge of the exact number of positive and negative examples in the test set ($n_1$ and $n_0$,
respectively).
\begin{proposition}
Let $\mathcal{D}$ be a dataset with labels $y_1,\ldots,y_n$, of which $n_0$ are negative and $n_1$ are positive.
Let $\hat{y}_1,\ldots,\hat{y}_n$ be a contestant's real-valued guesses for the labels of $\mathcal{D}$
such that $\hat{y}_i=\hat{y}_j \iff i=j$. Let $c=f(y_{1:n}, \hat{y}_{1:n})$ denote the AUC achieved by the real-valued guesses
with respect to the ground-truth labels. For any positive integer $k \leq n_0$, if $c > 1-\frac{1}{n_1}+\frac{k}{n_0 n_1}$,
then the first $k$ examples according to the rank order of $\hat{y}_1,\ldots,\hat{y}_n$ must be negatively labeled. Similarly,
for any positive integer $k \leq n_1$, if $c > 1-\frac{1}{n_0}+\frac{k}{n_0 n_1}$,
then the last $k$ examples according to the rank order of $\hat{y}_1,\ldots,\hat{y}_n$ must be positively labeled.
\end{proposition}
\begin{proof}
Without loss of generality,
suppose that the indices are arranged such that the $\hat{y}_i$ are sorted, i.e., $\hat{y}_1 < \ldots < \hat{y}_{n}$.
Suppose, by way of contradiction, that $m$ of the first $k$ examples were \emph{positively} labeled, where $1 \leq m \leq k$. For each such possible $m$,
we can find at least $m (n_0 - k)$ pairs that are misclassified according to the real-valued guesses
by pairing each of the $m$ positively labeled examples within
the index range $\{ i:\ 1\leq i \leq k\}$ with each of $(n_0 - k)$ negatively labeled examples within the index range $\{j:\ k+1 \leq j \leq n\}$.
In each such pair $(i,j)$, $y_i=1$ and $y_j=0$, and yet $\hat{y}_i < \hat{y}_j$, and thus the pair is misclassified. The minimum
number of misclassified pairs, over all $m$ in the range $\{1,\ldots, k\}$, is clearly $n_0 - k$ (for $m=1$).
Since there are $n_0 n_1$ total pairs in $\mathcal{D}$ consisting of one positive and one negative example, and since the AUC is maximized
when the number of misclassified pairs is minimized, then the maximum AUC that could be achieved
when $m\geq 1$ of the first $k$ examples are positively labeled is
\[
1 - \frac{n_0 - k}{n_0 n_1} = 1 - \frac{1}{n_1} + \frac{k}{n_0 n_1}
\]
But this is less than $c$. We must therefore conclude that $m=0$, i.e., all of the first $k$ examples are negatively labeled.

The proof is exactly analogous for the case when $c > 1-\frac{1}{n_0}+\frac{k}{n_0 n_1}$.
\end{proof}

\subsection{Example}
Suppose that a contestant's real-valued guesses $\hat{y}_1,\ldots,\hat{y}_n$ achieve an AUC of $c=0.985$
on a dataset containing $n_0=45$ negative and $n_1=55$ positive examples, and that $n_0$ and $n_1$ are known to him/her.
Then the contestant can conclude that the labels of the first (according to the rank order of $\hat{y}_1,\ldots,\hat{y}_n$) $7$ examples
\emph{must} be negative and the labels of the last $17$ examples \emph{must} be positive.

\section{Attack 2: Integration over Satisfying Solutions}
The second attack that we describe treats the AUC reported by the oracle as an observed random variable in a standard supervised learning
framework. In contrast to Attack 1, no prior knowledge of $n_0$ or $n_1$ is required.
Note that the attack we describe uses only a \emph{single} oracle query to improve an existing set of real-valued guesses. More sophisticated
attacks might conceivably refine the contestant's guesses in an iterative fashion using multiple queries.
\noindent {\bf Notation}: In this section only, capital letters denote random variables and lower-case letters represent instantiated values.

\begin{figure}
\begin{center}
\includegraphics[width=3.00in]{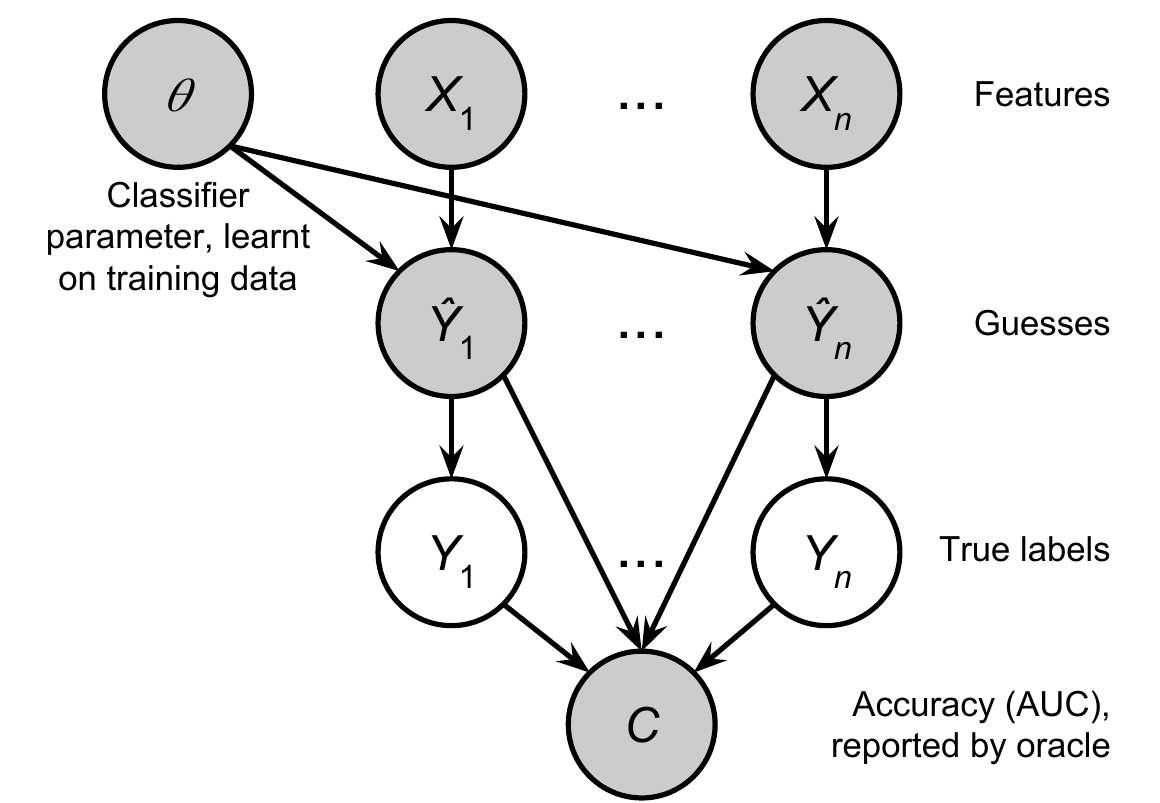}
\end{center}
\caption{Without node $C$, this graphical model shows a standard supervised learning problem: after estimating $\theta$ (on training data, not shown),
the test labels $Y_1,\ldots,Y_n$ can be estimated from feature vectors $X_1,\ldots,X_n$, and then submitted to the organizers of
the competition. Node $C$ represents the contestant's accuracy (AUC), which is often provided by an oracle and can be leveraged
to improve the guesses for the test set labels. Only the shaded variables are observed.
}
\label{fig:logistic_regression}
\end{figure}
Consider the graphical model of Fig.~\ref{fig:logistic_regression}, which depicts
a test set containing $n$ examples where each example $i$ is described by a vector $X_i\in\mathbb{R}^m$ of $m$ features (e.g., image pixels).
Under the model, each label $Y_i \in \{0,1\}$ is generated according to $P(Y_i=1\ |\ x_i, \theta) = g(x_i, \theta)$,
where $g: \mathbb{R}^m \times \mathbb{R}^m \rightarrow [0,1]$
could be, for example, the sigmoid function of logistic regression and $\theta \in \mathbb{R}^m$
is the classifier parameter. Note that this is a standard probabilistic discriminative
model -- the only difference is that we have created an intermediate variable $\hat{Y}_i \in [0,1]$ to represent
$g(x_i, \theta)$ explicitly for each $i$. Specifically, we define:
\begin{eqnarray*}
P(\hat{y}_i\ |\ x_i, \theta) &=& \delta(\hat{y}_i - g(x_i, \theta))\\
P(Y_i=1\ |\ \hat{y}_i) &=& \hat{y}_i
\end{eqnarray*}
where $\delta$ is the Dirac delta function.

The classification parameter $\theta$ can be estimated
from training data (not shown), and thus we consider it to be observed. Using $X_{1:n}$ and an estimate for $\theta$,
the contestant can then compute $\hat{Y}_{1:n}$ and submit these as his/her guesses to the competition organizers.
The question is: if the competition allows access to an oracle that reports $C$ (i.e., variable $C$ is observed), how can this additional
information be used? Since the AUC $C$ is a deterministic function (Eq.~\ref{eqn:auc}) of $y_{1:n}$ and $\hat{y}_{1:n}$,
we have:
\[
P(c\ |\ y_{1:n}, \hat{y}_{1:n}) = \delta(f(y_{1:n}, \hat{y}_{1:n}) - c)
\]
Then, according to standard Bayesian inference, the contestant can use $C$ to update
his/her posterior estimate of each label $Y_i$ as follows:
\begin{eqnarray}
\nonumber \lefteqn{P(y_i \ |\ \hat{y}_{1:n}, c)} &&\\
\nonumber  &=& \sum_{y_{\neg i}} P(y_{1:n}\ |\ \hat{y}_{1:n}, c) \\
\nonumber   && \textrm{where $y_{\neg i}$ refers to $y_1,\ldots,y_{i-1},y_{i+1},\ldots,y_n$.} \\
\nonumber   &=& \sum_{y_{\neg i}} \frac{P(c\ |\ y_{1:n}, \hat{y}_{1:n}) P(y_{1:n}\ |\ \hat{y}_{1:n})}
                             {P(c\ |\ \hat{y}_{1:n})} \\
\nonumber   &\propto& \sum_{y_{\neg i}} \left[ P(c\ |\ y_{1:n}, \hat{y}_{1:n}) \prod_{j} P(y_j\ |\ \hat{y}_j) \right]\\
\nonumber   &=& \sum_{y_{\neg i}} \delta(f(y_{1:n}, \hat{y}_{1:n}) - c) \prod_{j} P(y_j\ |\ \hat{y}_j) \\
            &=& \sum_{\substack{\text{$y_{\neg i}:$}\\\text{$f(y_{1:n},\hat{y}_{1:n}) = c$}}} \prod_{j} P(y_j\ |\ \hat{y}_j)
				   \label{eqn:attack2}
\end{eqnarray}
In other words, to compute the (unnormalized) posterior probability that $Y_i=y_i$ given $C=c$,
simply find all label assignments to the \emph{other} variables $Y_1,\ldots,Y_{i-1},Y_{i+1},\ldots,Y_n$
such that the AUC is $c$, and then compute the sum of the likelihoods $\prod_j P(y_j\ |\ \hat{y}_j)$ over all such assignments.

\subsection{Simulation}
As a proof-of-concept of the algorithm described above, we conducted a simulation 
on a tiny dataset of $n=16$ examples. In particular, we let $g(x,\theta)=(1+\exp(-\theta^\top x))^{-1}$
(sigmoid function for logistic regression), and we sampled $\theta$ and $X_1,\ldots,X_n$
from an $m$-dimensional Normal distribution with zero mean and diagonal unit covariance.
In our simulations, the contestant does not know the value of $\theta$ but can estimate it from a training dataset containing $k$ examples
(with $L_2$ regularization of $1$). In each simulation,
the contestant computes $\hat{Y}_{1:n}$ from its estimate of $\theta$ and the feature vectors $X_{1:n}$. The contestant then submits $\hat{Y}_{1:n}$ as guesses
to the oracle, receives the AUC score $C$, and then computes $P(Y_1=1\ |\ \hat{y}_{1:n}, c), \ldots, P(Y_n=1\ |\ \hat{y}_{1:n}, c)$. The contestant then submits
these posterior probabilities as its \emph{second} set of guesses and receives an updated AUC score $C'$. After each simulation run, we record the accuracy gain
$C'-C$. By varying $k\in\{1,\ldots,20\}$, $m\in\{4,5,\ldots,16\}$, and averaging over $50$ simulation runs per $(m,k)$ combination, we can then compute the expected
accuracy gain $\Delta AUC$ (i.e., $C'-C$) as a function of the initial AUC ($C$).

\begin{figure}
\begin{center}
\includegraphics[width=3.3in]{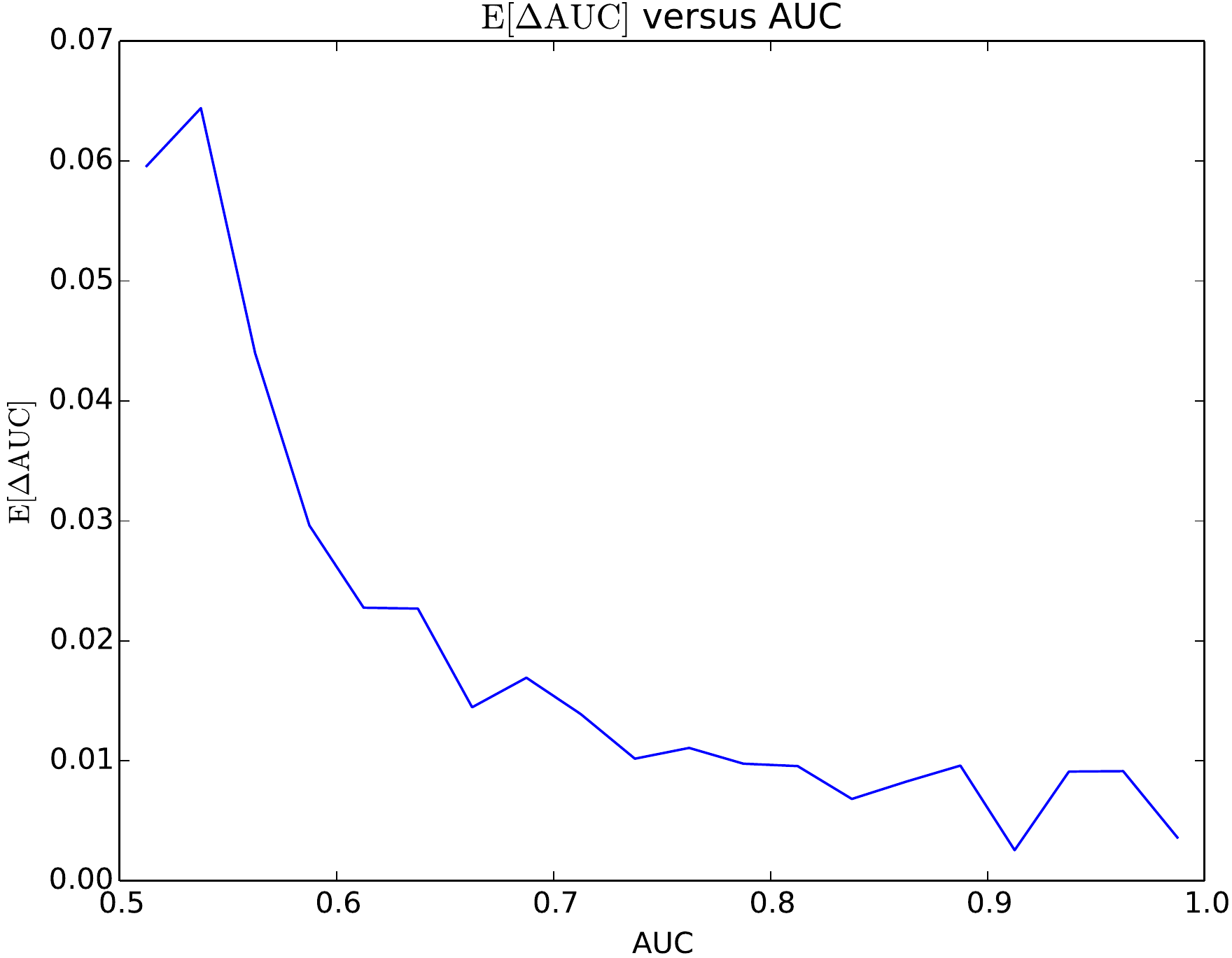}
\end{center}
\caption{Tiny simulation of how exploiting knowledge of the AUC of a set of guesses can improve
accuracy, as a function of existing accuracy.}
\label{fig:simulation}
\end{figure}
{\bf Results}: The graph in Fig.~\ref{fig:simulation} indicates that, for a wide range of starting AUC values
$C$, a small but worthwhile average increase in accuracy can be achieved, particularly when $C$ is closer to $0.5$.

\subsection{Tractability}
In the simulation above,
we used a brute-force approach when solving Eq.~\ref{eqn:attack2}: we created a list of all $2^n$ possible binary tuples $(y_1,\ldots,y_n)$
and then simply selected those tuples that satisfied $f(y_{1:n}, \hat{y}_{1:n})=c$.
However, if the number of such tuples were small, and if one could efficiently enumerate over them, then the attack 
would become much more practical.
This raises an important question: for a dataset of size $n$ and a fixed set of
real-valued guesses $\hat{y}_1,\ldots,\hat{y}_n$, are there certain AUC values for which  the number of possible
binary labelings is sub-exponential in $n$? We investigate this question in the next section.

\section{The Number of Satisfying Solutions Grows Exponentially in $n$ for Every AUC $c\in (0,1)$}
\begin{figure*}
\begin{center}
\includegraphics[width=6in]{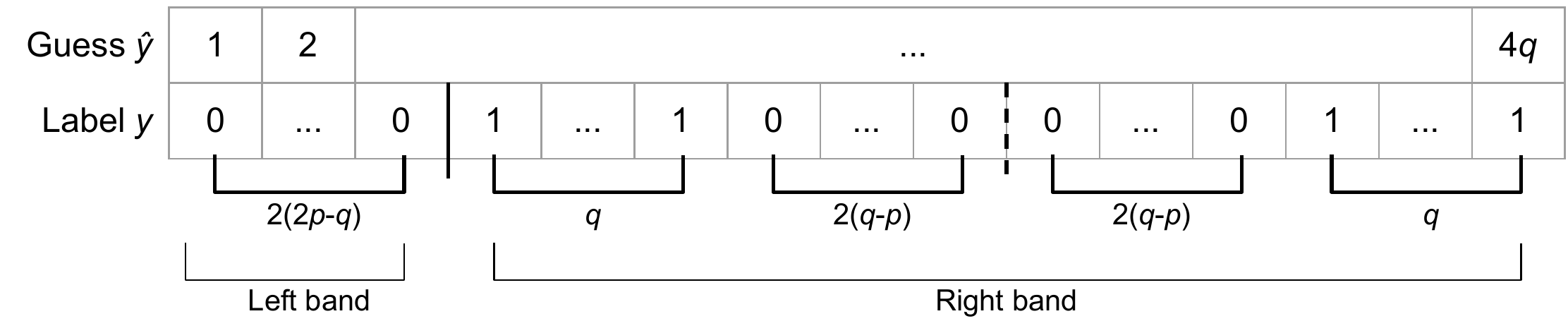}
\end{center}
\caption{Construction of a binary labeling for which the AUC is $c$, for any $c=\frac{p}{q}$ such that $0.5 \leq c < 1$.}
\label{fig:construction}
\end{figure*}
In this section we prove that the number of tuples $(y_1,\ldots,y_n)$ for which a contestant's guesses
achieve a fixed AUC $c$ grows exponentially in $n$, for every $c\in (0,1)$. (Note that this is different from proving that,
for a dataset of some fixed size $n$, the number of tuples $(y_1,\ldots,y_n)$ for which a contestant's
guesses achieve some AUC $c$ is exponential in $n$.)
Our proof is by construction: for any AUC $c=\frac{p}{q},\ p,q\in\mathbb{Z}^+,\ p<q$, 
we show how to construct a dataset of size $n=4q$ such that the number of satisfying binary labelings is at least $\left(2 - 2\left|c-0.5\right|\right)^{n/4}$.
Intuitively, this lower-bound grows more quickly for AUC values close to 0.5 than for AUC values close to 0 or 1.

First, we prove a simple lemma:

\begin{lemma}
\label{lemma}
Let $a,b,c \in\mathbb{Z}$ such that $a>b> 0$ and $c\geq 0$. Then $\frac{a+c}{b+c} \leq \frac{a}{b}$.
\end{lemma}
\begin{proof}
By way of contradiction, suppose $\frac{a+c}{b+c} > \frac{a}{b}$. Then
\begin{eqnarray*}
(a+c)b &>& a(b+c) \\
ab+bc  &>& ab + ac \\
bc     &>& ac \\
\end{eqnarray*}
which implies $b>a$.
\end{proof}

\begin{proposition}
Let $\mathcal{D}$ be a dataset consisting of $n=4q$ examples for some $q\in\mathbb{Z}^+$, and
let $\hat{y}_1,\ldots,\hat{y}_n$ be a contestant's real-valued guesses for the binary labels of $\mathcal{D}$
such that $\hat{y}_i=\hat{y}_j \iff i=j$. Then
for any AUC $c$ such that $c=\frac{p}{q}, p,q\in \mathbb{Z}^+$ and $p<q$, the number of distinct binary labelings $y_1,\ldots,y_n$ 
for which $f(y_{1:n},\hat{y}_{1:n})=c$ is at least $\left(2 - 2\left|c-0.5\right|\right)^{n/4}$.
\end{proposition}
\begin{proof}
Without loss of generality,
suppose that the indices are arranged such that the $\hat{y}_i$ are sorted, i.e., $\hat{y}_1 < \ldots < \hat{y}_{n}$. Since the
AUC is invariant under monotonic transformations of the real-valued guesses, we can represent each $\hat{y}_i$ simply by its
index $i$. Since $c$ is a fraction of pairs that are correctly classified, we can write it as $p/q$ for positive integers $p,q$.
We will handle the cases $c \geq 0.5$ and $c< 0.5$ separately.

{\bf Case 1 ($0.5 \leq c < 1$)}:
Construct a dataset of size $n=4q$ as shown in Fig.~\ref{fig:construction}: the first $2(2p - q)$ entries are negative
examples, and of the remaining $2(3q - 2p)$ entries (which we call the ``right band''), $2q$ are positive and $4(q-p)$ are negative. Moreover, the right band
is arranged \emph{symmetrically} in the following sense: for each $i\in\{2(2p-q)+1,\ldots,4q\}$, $y_i=1\iff y_j=1$ where
$j=2(2p-q) + (4q-i+1)$.

Given this construction, we must calculate how many pairs containing one positive and one negative example are
correctly and incorrectly classified. Note that each of the first $2(2p-q)$ negative examples in the left band can be paired with each of the
$2q$ positive examples in the right band, and that each of these positive
examples has a higher $\hat{y}$ value than the negative examples; hence, each of these $2q(2)(2p-q)=8pq-4q^2$ pairs is
classified correctly. The only remaining pairs of examples occur within the right band. To calculate the number of correctly/incorrectly
classified pairs in the right band, we exploit the fact that it is symmetric:
For any index pair $(i,j)$ where $i,j\in\{2(2p-q)+1,\ldots,4q\}$, where $y_i=0$ and $y_j=1$, and where $i<j$ (and hence
$\hat{y}_i < \hat{y}_j$), we can find exactly one other pair $(i',j')$ for which
$y_{i'}=0$ and $y_{j'}=1$ and for which $i'>j'$. Hence, within the right band,
the numbers of correctly and incorrectly
classified pairs are equal. Since there are $2q(4)(q-p)=8q^2 - 8pq$ pairs within this band total, then $4q^2 -4pq$
are classified correctly and $4q^2-4pq$ are classified incorrectly.

Summing the pairs of examples within the right band and the pairs between the left and right bands, we have
$8pq-4q^2 + 8q^2 - 8pq = 4q^2$ pairs total. The number of correctly classified pairs is
$4q^2 - 4pq + 8pq - 4q^2 = 4pq$, and thus the AUC is $4pq/(4q^2) = p/q =c$, as desired.

Now that we have constructed a \emph{single} labeling of size $n=4q$ for which the AUC is $c$, we can construct many more
simply by varying the positions of the $2q$ positive entries within the right band of $2(3q-2p)$ entries.
To preserve symmetry, we can vary the positions of half of the positive examples within half of the right band and then simply ``reflect'' the positions onto
the other half. In total, the number of choices is:
\begin{eqnarray*}
\lefteqn{{3q-2p \choose q}} &&\\
  &=& \frac{(3q-2p)!}{q!(3q-2p-q)!} \\
  &=& \frac{(3q-2p)!}{q!(2q-2p)!} \\
  &=& \frac{(3q-2p)(3q-2p-1)\cdots(2q-2p+2)(2q-2p+1)}{q(q - 1) \cdots (2)(1)} \\
\end{eqnarray*}
We now apply Lemma \ref{lemma} and the fact that the numerator and denominator both contain $q$ terms:
\begin{eqnarray*}
\lefteqn{{3q-2p \choose q}} &&\\
  &\geq& \left( \frac{3q-2p}{q} \right)^{q} \\
  &=& \left( \frac{3q-2p}{q} \right)^{n/4} \\
  &=& \left( 3 - 2c \right)^{n/4} \\
\end{eqnarray*}

{\bf Case 2 ($0<c<0.5$)}: The proof is analogous except that we ``flip'' the AUC around $0.5$ and correspondingly ``flip''
the labels left-to-right. Let $r=q-p$ (so that $r/q = 1 - p/q$). Then, we form a similar construction as above,
except that the left sequence of all negative examples
is moved to the right. Specifically, we create a symmetric sequence of length $2(3q-2r)$ such that $2q$ examples are positive and
$4(q-r)$ are negative. We then append $2(2r-q)$ entries to the right that are all negative. This results in
\[
\frac{1}{2}\left(2q \times 4(q-r)\right) = 4q^2 - 4qr
\]
pairs that are classified correctly and
\begin{eqnarray*}
\lefteqn{2(2r-q) \times 2q + \frac{1}{2}\left(2q \times 4(q-r)\right)} && \\
  &=& 8qr - 4q^2 + 4q^2 - 4qr\\
  &=& 4qr
\end{eqnarray*}
pairs that are classified incorrectly. In total, there are $4qr + 4q^2 - 4qr = 4q^2$ pairs, so that the AUC
is $(4q^2 - 4qr) / 4q^2 = 1 - r/q = p/q$, as desired.

Analogously to above, we can form
\begin{eqnarray*}
{3q-2r \choose q} &=& \frac{(3q-2r)!}{q!(2q-2r)!} \\
  &\geq& \left( \frac{3q-2r}{q} \right)^{n/4} \\
  &=& \left( \frac{3q-2(q-p)}{q} \right)^{n/4} \\
  &=& \left( \frac{q+2p}{q} \right)^{n/4} \\
  &=& \left( 1+2c \right)^{n/4} \\
\end{eqnarray*}
symmetric constructions of the left band.

Combining Case 1 and Case 2, we find that the number of binary labelings for which the AUC is $c$ is at least
\[
\left(2 - 2\left|c-0.5\right|\right)^{n/4}
\]
for all $0<c<1$.
\end{proof}

\section{Conclusion}
In this paper we have examined properties of the Area Under the ROC Curve (AUC) that can enable a contestant of a data-mining
contest to exploit an oracle that reports AUC scores to illegitimately attain higher performance.
We presented two simple attacks: one whereby a contestant whose guesses already achieve high accuracy can infer, with
complete certainty, the values of a few of the test set labels; and another whereby a contestant can harness the oracle's AUC information
to improve his/her guesses using standard Bayesian inference. To our knowledge, our paper is the first to formally investigate these kinds of attacks.

The practical implications of our work are mixed: On the one hand, we have provided proofs-of-concept that systematic exploitation
of AUC oracles is possible, which underlines the importance of taking simple safeguards such as (a) adding noise to the output of the oracle,
(b) limiting the number of times that a contestant may query the oracle, and (c) not re-using test examples across competitions.
On the other hand, we also provided evidence -- in the form of a proof  that the number of binary labelings for which a set of guesses
attains an AUC of $c$
grows exponentially in the test set size $n$ -- that brute-force probabilistic inference to improve one's guesses is intractable
except for tiny datasets. It is conceivable that some approximate inference algorithms might overcome this obstacle.

As data-mining contests continue to grow in number and importance, it is likely that more creative exploitation -- e.g.,
attacks that harness multiple oracle queries instead of just single queries --
will be attempted.
It is even possible that the focus of effort in such contests might shift from developing effective machine learning classifiers
to querying the oracle strategically, without training a useful classifier at all.
With our paper we hope to highlight an important potential problem in the machine learning community.

\bibliographystyle{alpha}
\bibliography{paper}

\end{document}